\documentclass[a4paper,UKenglish,cleveref, autoref, thm-restate]{lipics-v2021}
\usepackage{hyperref}
\usepackage{xcolor,colortbl}
\usepackage{mathtools}
\usepackage{caption}
\usepackage{enumitem}
\usepackage{dsfont}
\usepackage{algorithm}
\usepackage{multicol}
\usepackage{svg}
\usepackage{bm}
\usepackage[noend]{algpseudocode}
\usepackage{amsmath}
\usepackage{breqn}
\usepackage{float}
\algnewcommand\algorithmicforeach{\textbf{for each:}}
\algnewcommand\ForEach{\item[ \algorithmicforeach]}
\algdef{S}[FOR]{ForEach}[1]{\algorithmicforeach\ #1\ \algorithmicdo}
\algtext*{EndWhile}
\algtext*{EndIf}

\title{Reducing False Discoveries in Statistically-Significant Regional-Colocation Mining: A Summary of Results} 

\titlerunning{Reducing False Discoveries in Statistically-Significant Regional-Colocation Mining} 

\author{Subhankar Ghosh}{Department of Computer Science \& Engineering, University of Minnesota, Minneapolis, MN, USA}{ghosh117@umn.edu}{}{}
\author{Jayant Gupta}{Department of Computer Science \& Engineering, University of Minnesota, Minneapolis, MN, USA}{gupta423@umn.edu}{}{}
\author{Arun Sharma}{Department of Computer Science \& Engineering, University of Minnesota, Minneapolis, MN, USA}{sharm485@umn.edu}{}{}
\author{Shuai An}{Department of Economics, University of Minnesota, Minneapolis, MN, USA}{an000033@umn.edu}{}{}
\author{Shashi Shekhar}{Department of Computer Science \& Engineering, University of Minnesota, Minneapolis, MN, USA}{shekhar@umn.edu}{}{}

\authorrunning{S., Ghosh et al.} 

\Copyright{Subhankar Ghosh, Jayant Gupta, Arun Sharma, Shuai An and Shashi Shekhar} 



\category{} 

\relatedversion{} 

\supplement{}

\funding{This material is based upon work supported by the National Science Foundation under Grants No. 2118285, 2040459, 1901099, and 1916518.}

\acknowledgements{We also thank Kim Koffolt, Yash Travadi, and the Spatial Computing Research Group for valuable comments and refinements.}

\EventEditors{John Q. Open and Joan R. Access}
\EventNoEds{2}
\EventLongTitle{42nd Conference on Very Important Topics (CVIT 2016)}
\EventShortTitle{CVIT 2016}
\EventAcronym{CVIT}
\EventYear{2016}
\EventDate{December 24--27, 2016}
\EventLocation{Little Whinging, United Kingdom}
\EventLogo{}
\SeriesVolume{42}
\ArticleNo{23}

\begin{document}
\nolinenumbers
\maketitle

\begin{abstract}
Given a set \emph{S} of spatial feature types, its feature instances, a study area, and a neighbor relationship, the goal is to find pairs $<$a region ($r_{g}$), a subset \emph{C} of \emph{S}$>$ such that \emph{C} is a statistically significant regional-colocation pattern in $r_{g}$. This problem is important for applications in various domains including ecology, economics, and sociology. The problem is computationally challenging due to the exponential number of regional colocation patterns and candidate regions. Previously, we proposed a miner \cite{10.1145/3557989.3566158} that finds statistically significant regional colocation patterns. However, the numerous simultaneous statistical inferences raise the risk of false discoveries (also known as the multiple comparisons problem) and carry a high computational cost. We propose a novel algorithm, namely, multiple comparisons regional colocation miner (MultComp-RCM) which uses a Bonferroni correction. Theoretical analysis, experimental evaluation, and case study results show that the proposed method reduces both the false discovery rate and computational cost. 
\end{abstract}

\section{Introduction}
Regional-colocation patterns are (study sub-area $R$, feature-type subset $C$) pairs such that instances of feature-types in $C$  often are present in $R$ in close proximity. Given a set \emph{S} of spatial features (e.g., coffee shops, restaurants), their feature instances, a study area, and a neighbor relationship (e.g., geographic proximity), the goal is to identify pairs  $<$region $r_{g}$, subset \emph{C} of \emph{S}$>$ such that instances of \emph{C} are statistically significant in that region $r_{g}$. Figure \ref{fig:ColocationToyExample}(a) shows a set of instances input into a regional-colocation miner, consisting of three different spatial feature types, a neighborhood relation between feature instances, and a space partitioning. Figure \ref{fig:ColocationToyExample} (b), shows the set of statistically significant regional colocations identified after significance testing (described in Section \ref{sec:stat_sig_test_basic_concepts}). The output is a pair of regional colocations: $r_{1}$ showing a strong regional colocation between all three features (i.e., $f_{A}$, $f_{B}$, and $f_C$) and $r_{2}$ showing a strong regional-colocation between two features (i.e., $f_{A}$ and $f_{B}$). The rest of the area within the map shows less spatial interaction (low participation index) between these features.

\begin{figure}[thbp]
\centering
\includegraphics[width=0.80\textwidth]{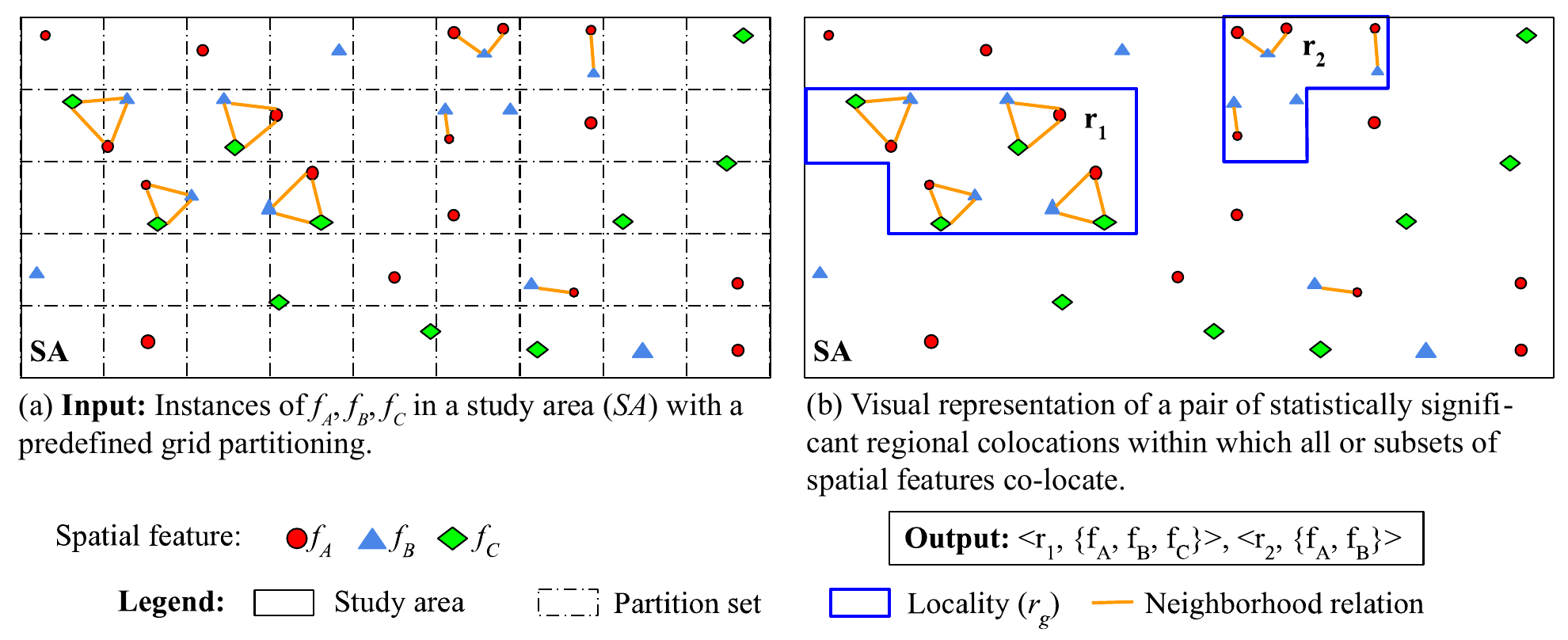}
\caption{Regions where \textbf{all} or \textbf{subsets} of $f_A$, $f_B$ and $f_C$ significantly co-locate in the study area}\label{fig:ColocationToyExample}
\end{figure}

The problem of mining statistically significant regional-colocation patterns is societally important with applications in retail, public health, ecology, public security, transportation, etc. For example, retail establishments (e.g., fast food chains and coffee shops) often colocate to reach each other's customers. Thus, finding statistically significant regional colocation patterns among competing retail stores has tremendous value for retail analysis. When identifying colocation patterns in societal domains, it's important to minimize the chance of false discoveries. A famous historical example was between $1900$ and $1904$ when urban districts of San Francisco experienced an outbreak of bubonic plague, resulting in $119$ deaths. The federal and state authorities falsely identified the victims' ethnicity as a highly correlated feature to the plague. This false discovery brought an immense adverse impact on San Francisco's management of the plague. Even when we don't unfairly stigmatize groups or regions, false discoveries waste money, and resources. Comparing the city's response to the same plague between $1907$ and $1908$, where rats were correctly identified as a highly correlated feature and the plague was swiftly contained, the negative impact of false discovery was even more strongly felt \cite{risse1992long}. Table \ref{tab:App_Domain} provides application domains and use cases.

\begin{table}[tbhp]
{
\footnotesize
\caption{Regional-colocation applications.}\label{tab:App_Domain}
\begin{tabular}{|p{1.85cm}|p{11cm}|}
 \hline
 \textbf{Application Domain} & \textbf{Example} \\ \hline
 Retail & $<$China, \{McDonald’s and KFC\}$>$,  $<$USA, \{McDonald’s and Jimmy John's\}$>$ \\\hline
 Public Health & $<$Ports, \{Plague and rats\}$>$, $<$Middle East, \{Middle East Respiratory Syndrome (MERS) in 2012 and MERS-CoV\}$>$ \\ \hline
 Ecology & $<$Indian/Pacific Ocean, \{Anemone and Clownfish\}$>$, $<$Nile River delta, \{Nile Crocodile and Egyptian Plover\}$>$\\ \hline
 Public Safety & $<$Region around bars, \{Assault crimes and drunk driving\}$>$ \\ \hline
 Transportation & $<$Near bus depots, \{High $NO_{x}$ concentrations and buses\}$>$\\ \hline
\end{tabular}
}
\end{table}

The problem of statistically significant regional-colocation pattern detection ($SSRCPD$) is computationally challenging due to the following reasons:  (1) Significance testing in this problem requires considering multiple statistical inferences simultaneously which leads to an increase in Type-$I$ error (i.e false discoveries). (2) There is an exponential number of candidate regional patterns, e.g., the dataset used in the case study (Section \ref{subsec:case-study}) consists of $1473$ different retail brands and their locations in Minnesota, resulting in $2^{1473}$ different candidate patterns. (3) Spatial partitioning approach would lead to an infinite number of candidate region subsets.
\begin{figure}
	\centering
	\includegraphics[width=0.45\linewidth]{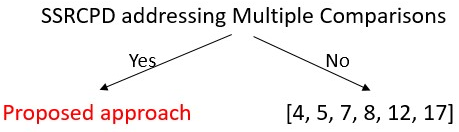}
	\caption{Comparison with Related Work}
	\label{fig:relatedwork}
 \vspace{-10pt}
\end{figure}

Figure \ref{fig:relatedwork} shows a decision tree that distinguishes our manuscript from previous works, where $SSRCPD$ refers to $Statistically\ significant\ regional\ colocation\ pattern\ detection$. Earlier work on regional-colocation pattern detection either uses data unaware space partitioning (e.g., Quadtree \cite{Celik2007ZonalCP, li2018local}) or clustering of colocation instances \cite{Deng2017MultilevelMF, Eick2008FindingRC}. However, these techniques lack statistical significance testing and depend on input parameters (e.g., participation index threshold) which may vary geographically. Statistically significant global colocation mining was introduced by \cite{barua2013mining}, while statistically significant regional colocation mining was first explored in \cite{10.1145/3557989.3566158}. In \cite{10.1145/3557989.3566158} we proposed $SSRCM$ which utilizes a subgraph enumeration approach to detect statistically significant regional colocation patterns where the regions would be composed of one or more contiguous atomic partitions (smallest region within which a candidate pattern is statistically significant). This algorithm was expensive because expanding the region within which the pattern was statistically significant required recalculating the $p$-$value$. Since detecting statistically significant regional colocation patterns requires performing multiple simultaneous statistical inferences, this results in the multiple comparisons problem \cite{rupert2012simultaneous}, which risks false discoveries (a.k.a. Type-I errors). The problem results in a rapid increase in the probability of Type-I error as the number of partitions increases. To address the multiple comparisons problem, we propose a robust statistically significant regional colocation miner (MultComp-RCM) using a Bonferroni correction \cite{bonferroni1936teoria}. The proposed approach recommends stricter $p$-$values$ to reduce false discoveries (Type-I errors), thus setting an upper bound on the overall significance level ($\alpha$, which is 0.05 for a 95\% statistical confidence).

\noindent\textbf{Contributions:} 
\begin{itemize}[noitemsep,topsep=5pt]
    \vspace*{-1.0mm}
    \item We proposed a new approach Multiple comparisons regional colocation miner (MultComp-RCM) to reduce false positives using a well-established statistical technique for multiple comparisons correction, the Bonferroni test.
    \item The paper provides a comparative analysis showing that the proposed MultComp-RCM is computationally more efficient than SSRCM.
    \item The paper describes a sensitivity analysis using synthetic data which shows that MultComp-RCM requires an increasingly smaller number of significance tests and participation index computations for an increasing number of regions.
    \item We proposed a case study on retail establishments in Minnesota using the Safegraph POI dataset \cite{10.1145/3557989.3566158}. The proposed method discovers new regional-colocation patterns involving fast food and coffee retailer feature-type subsets in a Minnesota counties study area. We also confirm that the Bonferroni correction in our method reduces false discoveries.
\end{itemize}

\noindent\textbf{Scope:} For simplicity, this paper focuses on regional-colocation patterns consisting of two or three different features. In our case study, we enumerated regions based on a contiguous collection of counties. Nevertheless, this work can be extended to different types of regions (e.g., ports). We also do not consider segregation patterns (negative spatial interaction) or the temporal aspects of the patterns.

\noindent\textbf{Organization:} The paper is organized as follows. Section \ref{sec:Problem Definition} reviews basic concepts and formally defines the problem. In section \ref{sec:Methodology} we briefly review SSRCM and decribe the proposed approach (MultComp-RCM). Section \ref{sec:theoretical_analysis} gives a theoretical analysis of MultComp-RCM. We present the experimental evaluation in Section \ref{sec:Experiments} and a case study in Section \ref{subsec:case-study}. Section \ref{sec:related} briefly surveys related work and discussion. Section \ref{sec:Conclusion} concludes the paper with future work.

\section{Basic Concepts and Problem Definition.}
\label{sec:Problem Definition}
First, we review basic concepts related to colocation detection, statistical significance testing, and the multiple comparisons problem. Then, we formally define statistically significant regional colocation pattern detection.
\subsection{Colocation detection:}
\label{sec:colocation_detection}
In this section, we briefly introduce some taxonomy and the basic concept used to define colocation pattern detection with examples. The basic concepts are as follows:

A \textbf{feature instance} is a geo-located spatial entity which is a type of Boolean feature $f$ with a geo-reference point location $p$ (e.g., latitude, longitude), represented as $<f,p>$. Multiple instances of a feature are represented as $f_{i}$ and can be related to other feature instances $f_{j}$ via a \textbf{neighbor relation} $\mathcal{R}$. For example, geographic proximity is represented as $\mathcal{R}_{f_{i}, f_{j}}$ $\leq$ $\theta$, where $\theta$ is the neighbor relation threshold. In a \textbf{neighbor graph}, we represent features that satisfy such relations as a \textit{node} and their relationship as an $edge$.

A \textbf{colocation candidate} $C$ is a set of  features defined in the given study area ($SA$) or a sub-region ($r_{g}$) where $r_{g}$ $\in$ $SA$. For example, Figure \ref{fig:ColocationToyExample}(a) shows $17$ spatial objects of type $f_{A}$ (circle), $12$ spatial objects of type $f_{B}$ (triangle), and $9$ instances of colocation pattern \{$f_{A}$, $f_{B}$\}. An instance of a \textbf{colocation} satisfies the neighborhood relation $\mathcal{R}$ and forms a \textbf{clique}.

A \textbf{participation ratio ($pr$)} is the ratio of feature instances participating in a relation $\mathcal{R}$ to the total number of instances inside the study region $(SA)$. For a given colocation candidate $C$ and feature $f$, it is represented as $pr(f, C)$ as shown in Equation \ref{equation:PR}:
\begin{equation}
    \label{equation:PR}
    pr(f, C) = \frac{participating\_instances(f, C)}{instance(f)}
\end{equation}

For the feature instances shown in Figure \ref{fig:ColocationToyExample}(a) the participation ratio values for the relation $\{f_A, f_B\}$ are $pr(f_{A}, \{f_{A}, f_{B}\}) = \frac{9}{17}$ and $pr(f_{B}, \{f_{A}, f_{B}\}) = \frac{8}{12}$. Further, the participation ratio within a region ($r_{g}$) for a feature $f$ is defined as $pr(f, [r_{g}, C])$. For example, in Figure \ref{fig:ColocationToyExample} $pr(f_{A}, [r_{2}, \{f_{A}, f_{B}\}])$ and $pr(f_{B}, [r_{2}, \{f_{A}, f_{B}\}])$ in region $r_2$ and has the value $\frac{4}{4}$ and $\frac{3}{4}$ respectively.

A \textbf{participation index ($pi$)} is the minimal participation ratio of all feature types in a colocation candidate as described in Equation \ref{equation:PI_basic_concepts}:
\begin{equation}
    \label{equation:PI_basic_concepts}
    pi(C) = \underset{f \in C} {min}(pr(f, C))
\end{equation}

The participation index quantifies the spatial interaction within features. Figure \ref{fig:ColocationToyExample}(a) shows participation index of features $f_{A}, f_{B}$ which can be represented as $pi(\{f_{A}, f_{B}\})$ which is  $min(\frac{9}{17},\frac{8}{12})$ or $\frac{9}{17}$.
A \textbf{regional participation index} is the minimal participation ratio of all feature types in the colocation candidate $C$ within region $r_{g}$ as shown below
\begin{equation}
    pi([r_{g}, C]) = \underset{f \in C}{min}(pr(f, [r_{g}, C]))
    \label{PI_regional}
\end{equation}

\noindent For instance in Figure \ref{fig:ColocationToyExample}, $pi([r_2, \{f_{A}, f_{B}\}]) = min(\frac{4}{4}, \frac{3}{4}) = \frac{3}{4}$.

\textbf{Colocation patterns} \cite{shekhar2001discovering} is the set of prevalent colocation candidates (based on a prevalence measure, e.g. $pi$), i.e., candidates comprised of features having a high positive spatial interaction. A \textbf{regional-colocation pattern} \cite{li2018local} is a paired region ($r_{g}$) and colocation pattern ($C$), i.e., $<r_{g}, C>$ where the features in pattern $C$ have a high positive spatial interaction in $r_{g}$.


\subsection{Statistical Significance in Colocation Detection:}
\label{sec:stat_sig_test_basic_concepts} 
A statistically significant colocation determines whether an assigned positive spatial interaction between features is statistically significant or could have been observed if the features were in complete spatial randomness (CSR). Other properties in CSR are as follows:
\begin{itemize}[noitemsep,topsep=1pt]
	\item Every feature instance has an equal probability of existing at any point in the study area.
	\item The locations of any feature instances in the study area are independent of each other.
\end{itemize}

\noindent A \textbf{null hypothesis} $(H_{0})$ is a statement of `no effect' or `no difference'. In our problem, the null hypothesis represents the scenario under which there is no spatial interaction between the features in the dataset, i.e., their existence is completely independent of each other.

\noindent An \textbf{alternative hypothesis} $(H_{a})$ is a statement that is tested against a null hypothesis. In our problem, an alternative hypothesis represents the scenario under which there is a positive spatial interaction between the features in the dataset in a region of interest.

\noindent A \textbf{Type-I error} refers to the erroneous rejection of an actually true null hypothesis (or a false positive). In our problem, this would refer to incorrectly assigning a candidate regional-colocation pattern as statistically significant, even though there is a high probability of this pattern being found in CSR or $H_{0}$.

\noindent A \textbf{Type-II error} refers to the failure to reject a null hypothesis $(H_{0})$ that is actually false (or a false negative). This would translate into incorrectly assigning a candidate regional-colocation pattern as not statistically significant.

\noindent A \textbf{point distribution} is a collection of geo-distributed points referring to an event (e.g., road accident) in a spatial domain. A \textbf{point process} ($PP$) is a statistical process that defines the probability distribution of a point over a region. Point processes are essential for defining the null or alternative hypothesis for our statistical significance test. 

\noindent A \textbf{Poisson point process} is defined in a generalized space $S_{P}$ with intensity $\Lambda$ having the following properties:
\begin{enumerate}[noitemsep,topsep=5pt]
	\item The number of points in a bounded Borel set (bounded sets that can be constructed from open or closed sets by repeatedly taking countable unions and intersections) $B \subset S_{P}$ is a Poisson random variable with mean $\Lambda(B)$.
	\item The number of points in $n$ disjoint Borel sets forms $n$ independent random variables. This property results in independent scattering or complete independence.
\end{enumerate}

\noindent\textbf{Null hypothesis generation}:
\begin{itemize}[noitemsep,topsep=1pt]
    \item For an identical distribution, we generate an equal number of instances of each feature in every partition using summary statistics of the constituent features of the pattern. This ensures that the null hypotheses datasets (although in CSR) closely model the observed dataset in each atomic partition.
    \item For independence, we sample instances from a Poisson point process \cite{illian2008statistical}. To check for acceptable auto-correlation, we use a pair correlation function (PCF) or $g(d)$ up to a distance $d$, where $d$ is data-driven. When $g(d)>1$, it suggests there is clustering at a distance $d$ within the feature instances, while $g(d)=1$ represents CSR. 
\end{itemize}

\noindent \textbf{Statistical significance test}: Since the participation index ($pi$) is used to quantify the strength of spatial interaction, the objective is to determine the probability of a pattern's  $pi$ in the observed data. Let $pi_{\emptyset}(C)$ denote the participation index for pattern $C$ in the null hypothesis and $pi_{obs}(C)$ represent the participation index for candidate colocation \emph{C} in the observed data. Then, we compute the following probability \cite{barua2013mining}:
\begin{equation}
\label{PIObs}
    p = pr(pi_{\emptyset}(C) \geq pi_{obs}(C)) = \frac{R^{\geq pi_{obs}} + 1}{R + 1}
\end{equation}
where $R^{\geq pi_{obs}}$ represents the number of Monte Carlo simulations within which the participation index ($pi_{\emptyset}(C)$) for pattern $C$ is greater than in the observed data ($pi_{obs}(C)$) and $R$ refers to the total number of Monte Carlo simulations. If $p \leq \alpha$, we consider $pi_{obs}(C)$ as statistically significant at level $\alpha$.

\noindent \textbf{Regional statistical significance test}: To test for regional significance, we use simulated (i.e., computer-generated) candidate regions. For example, if we are trying to determine if $f_A$ and $f_B$ are statistically significant in locality $r_2$ (Figure \ref{fig:ColocationToyExample}(b)), we generate null hypothesis samples within the locality's boundaries and use the participation index result from each sample to perform the significance test for $r_g$. Figure \ref{fig:NH_Example}(a) and \ref{fig:NH_Example}(b) display two of the $R$ different null hypotheses which are used to compare the participation index of the regional-colocation pattern $[r_2, \{f_{A}, f_{B}\}]$ in the observed data. As shown, the participation index for the pattern $[r_2, \{f_{A}, f_{B}\}]$ is $1/3$ in both null hypotheses respectively. For a statistical confidence level of $95\%$ the following inequality should hold:
\begin{equation}
\centering
    \label{stat_confidence}
    \Big[ \sum_{r=1}^{R=99} \mathds{1}\bm{(}pi_{obs}([r_{g}, \{f_{A}, f_{B}\}]) \leq pi_{\emptyset_{r}}([r_{g}, \{f_{A}, f_{B}\}])\bm{)} \Big] < 5
\end{equation}
where $r_{g}$ is the region of interest, and $\mathds{1}$ is an indicator function. We can compute $R$ from $\alpha  = 0.05$ using $\alpha(R + 1) = 5$ \cite{besag1977simple}.
\begin{figure}[thbp]
\centering
\includegraphics[width=0.85\textwidth]{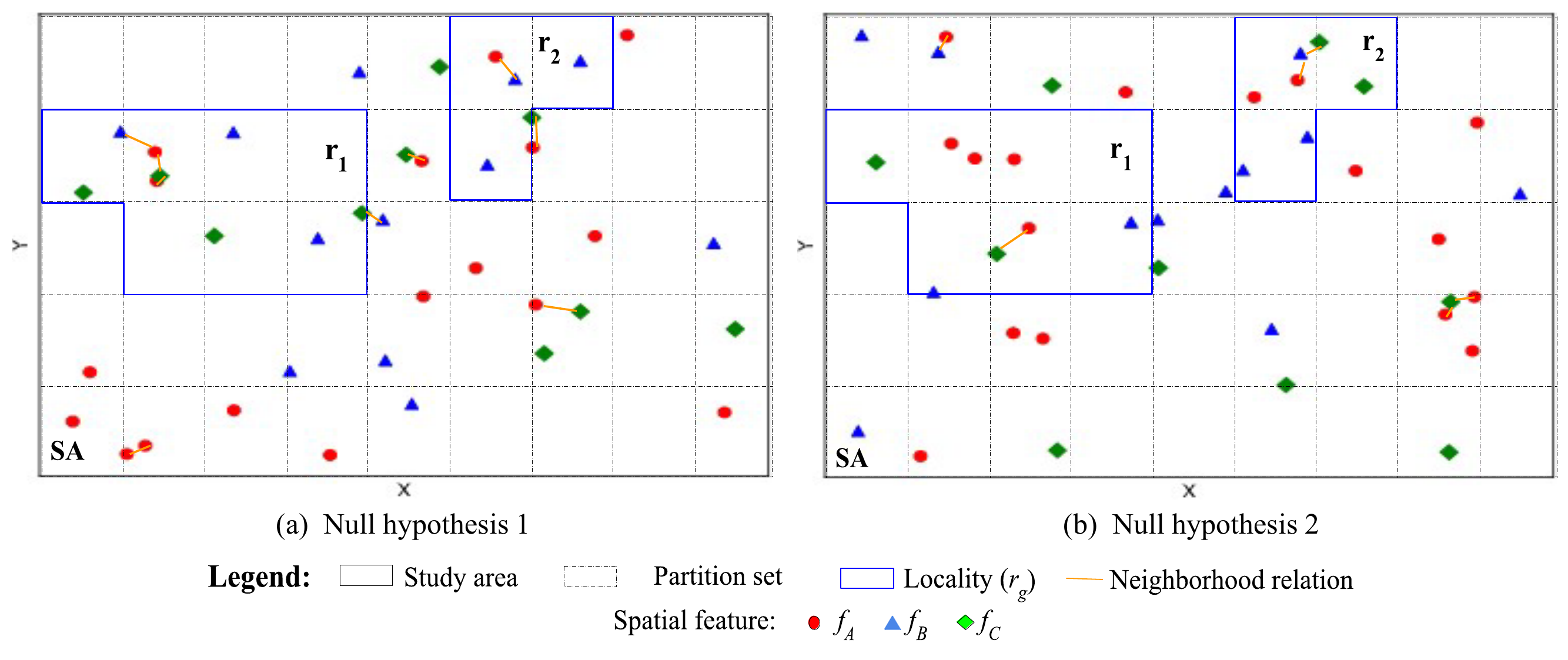}
\caption{Two example null hypotheses for significance testing of the observed data in Figure \ref{fig:ColocationToyExample}}\label{fig:NH_Example}
\end{figure}
The \textbf{multiple comparisons problem} \cite{rupert2012simultaneous} occurs when every inference in a set of statistical inferences simultaneously  has the potential to produce a discovery. The more inferences are made on a particular data set, the more likely it is to incorrectly reject the null hypothesis. Most techniques to address this problem require a stricter significance threshold for individual comparisons to compensate for the number of inferences being made. A stated confidence level generally applies to individual tests. It is often desirable to have a confidence level for a whole family of simultaneous tests.

The \textbf{Bonferroni correction \cite{bonferroni1936teoria}} is a method to address the multiple comparisons problem and the simplest method for reducing Type-I errors. It is a conservative method with a greater risk of failure to reject a false null hypothesis, thus resulting in Type-II errors.
\vspace*{-2.0mm}
\subsection{Formal problem formulation}
\label{sec:problem_definition}
\vspace*{-1.0mm}
The problem of statistically significant regional-colocation pattern detection is as follows:\newline
\textbf{Input:}
\begin{enumerate}[noitemsep,topsep=2pt]
    \vspace*{-0.5mm}
    \item A set ($F$) of spatial-features 
    \item $N$ geo-located spatial feature instances.
    \item A study area $S_{A}$ composed of space partitions (e.g., counties).
    \item A statistical significance level $\alpha$.
    \item A neighbor relationship ($\mathcal{R}$).
    \vspace*{-0.5mm}
\end{enumerate}
\textbf{Output:} Statistically significant regional-colocation patterns, $<r_{g}, C>$ where $C\subset F$.
\\\textbf{Objective:} Reducing Type-I error (false positives).\\
\textbf{Constraints:} Higher statistical confidence of output patterns.

\textbf{Reasoning behind problem output:} Testing for statistical significance on regional-colocation outputs ensures that spurious patterns aren't detected from the dataset. Otherwise, regions may be enumerated due to a high density of feature instances or spatial auto-correlation. In addition, significance testing for the union of many partitions leads to multiple statistical inferences. Due to the union of partitions, the probability of finding chance patterns within the bigger region (i.e., the union of partitions) is higher; this phenomenon is not accounted for by the $p$-$value$ threshold for a single partition. This leads to the multiple comparisons problem, resulting in a higher false discovery rate. In application domains related to regional-colocation pattern detection, reducing Type-I errors (false positives) takes higher priority over reducing Type-II errors (false negatives). These Type-II errors might result in missing the detection of certain patterns which might have a lower $p-value$.

In this situation, checking for a particular $\alpha$ level in the individual statistical inferences is insufficient. We also need to control the family-wise error rate which represents the probability of making one or more false discoveries (Type-I errors) \cite{rupert2012simultaneous}. We use the Bonferroni correction in MultComp-RCM to tackle this problem arising from multiple hypothesis tests. This conservative approach ensures that the pattern output has high statistical confidence while ignoring patterns that might have comparatively lower confidence, which is our primary objective. Another benefit of this method is the computational efficiency due to the smaller number of significance tests and participation index computations required as compared to the baseline \cite{10.1145/3557989.3566158}. The Bonferroni correction proposes stricter $p-value$ thresholds which might be a bottleneck for large scale applications, such as when dealing with hundreds of atomic partitions. This may also lead to a higher possibility of false negatives (Type-II errors).

\section{Methodology}
\label{sec:Methodology}
To keep the paper self-contained, we first briefly review the SSRCM, our previous statistically significant regional-colocation miner \cite{10.1145/3557989.3566158}, and a sub-routine on significance testing. We then describe the proposed approach in Section \ref{sec:proposed} and provide an example highlighting the computational cost savings of the new approach.

\subsection{Statistically Significant Regional-Colocation Miner}
\label{sec:baseline}

\textbf{Key idea:} In \cite{10.1145/3557989.3566158}, we started by considering partitions with at least $3$  instances of each feature which comprise the regional-colocation pattern.  This ensures that the features constituting the pattern all have a considerable presence in the enumerated partitions. We then use the regional statistical significance test as described in Algorithm \ref{algorithm1} to determine the atomic footprints of the pattern, i.e., statistically significant pattern within individual partitions. While computing the participation index, we limit our neighborhood to an empirically determined distance ($d$) to mine meaningful colocated features.

\begin{algorithm}[thp]
\caption{Significance testing} \label{algorithm1}
\footnotesize
\begin{flushleft}
\vspace{-10pt}
\hspace{\algorithmicindent}\textbf{Input:}\\
\hspace*{\algorithmicindent} - A spatial dataset \emph{S} consisting of features \{$f_{A}, f_{B}, ...$\}\\
\hspace*{\algorithmicindent} - A study area $(S_{A})$ and an atomic partition $r_{g} \subset S_{A}$\\
\hspace*{\algorithmicindent} - Statistical significance level $\alpha$\\
\hspace*{\algorithmicindent} - A candidate colocation pattern $C$\\
\hspace*{\algorithmicindent} - A set of $R$ Null hypotheses ($NH_{\emptyset}$) data each modelled as colocation $C$ in atomic partition $r_{g}$\\
\hspace*{\algorithmicindent} - Distance $d$ for participation index $(pi)$ calculation\\
\hspace*{\algorithmicindent}\textbf{Output:}\\
\hspace*{\algorithmicindent} 1. $<r_{g}, C>$ is significant or not \\
 \hspace*{\algorithmicindent} 2. $p$-$value_C^{r_g}$\\
\end{flushleft}
\vspace{-10pt}
\begin{algorithmic}[1]
    \Procedure{Significance Testing}{}
    \State{Statistically significant result $SSR_C^{r_g}$ $\leftarrow$ False}
            \State {Counter $R^{\geq pi_{obs}}$ $\leftarrow$ 0}
            \State {Calculate $pi_{obs}$ for $C$ at $d$ in $r_{g}$}
            \For{$i \in$ [$1,R$]}
                \State{Calculate the $pi_{\emptyset, i}$ of $C$ at $d$ in the $i^{th}$ $NH_{\emptyset}$}
                \If{$pi_{\emptyset, i} \geq pi_{obs}$}
                    \State $R^{\geq pi_{obs}} \leftarrow R^{\geq pi_{obs}} + 1$
                \EndIf
            \EndFor
                \State $p$-$value_{C}^{r_g}= \frac{R^{\geq pi_{obs}}+1}{R+1}$
                \If{$p$-$value_{C}^{r_g}$ $\leq$ $\alpha$}
                    \State {$SSR_C^{r_g}$ $\leftarrow$ True \Comment{(i.e., $<r_{g}, C>$ is statistically significant)}}
                \Else
                    \State{$SSR_C^{r_g}$ $\leftarrow$ False \Comment{(i.e., $<r_{g}, C>$ is not statistically significant)}}
                \EndIf
    \State {return $SSR_C^{r_g}$, $p$-$value_{C}^{r_g}$}
\EndProcedure
\end{algorithmic}
\end{algorithm}
\setlength{\textfloatsep}{5pt}

To find the union of partitions, we first form an undirected unweighted graph $(G=V, E)$ where each vertex $(V)$ refers to a partition within which the pattern is statistically significant and an edge $(E)$ between two $V$s represents a shared boundary between them. The graph representation allows the use of graph traversal algorithms (e.g., DFS) to find statistically significant regions that are the union of partitions in $V$.

We note that the union of two atomic footprints within which a candidate regional colocation pattern is statistically significant does not imply that the resultant footprint is a significant regional colocation pattern.  Thus, we must recompute the $pi$ for the candidate pattern in the new region and perform the significance test again.  As we progress along the edges of $G$, the final output is a larger region composed of contiguous atomic partitions such that the candidate pattern is statistically significant, both within the atomic partitions as well as in the region formed by the union of the output atomic footprints. This is represented by the  \textbf{largest connected component}. Algorithm \ref{algorithm2} provides the pseudo-code of $SSRCM$ to find statistically significant regional colocations.

\begin{algorithm}[h]
\caption{Statistically Significant Regional-Colocation Miner (SSCRM)} \label{algorithm2}
\footnotesize
\begin{flushleft}
\vspace{-10pt}
\hspace{\algorithmicindent}\textbf{Input:}\\
\hspace*{\algorithmicindent} - A Spatial dataset \emph{S} consisting of features \{$f_{A}, f_{B}, ...$\}\\
\hspace*{\algorithmicindent} - A study area $(S_{A})$ and a space partitioning $R_{g}$\\
\hspace*{\algorithmicindent} - Statistical significance level $\alpha$\\
\hspace*{\algorithmicindent} - Maximum pattern size $N$\\
\hspace*{\algorithmicindent} - Lower bound $LB$ (in meters)\\
\hspace*{\algorithmicindent} - Upper bound $UB$ (in meters)\\
\hspace*{\algorithmicindent}\textbf{Output:}\\
\hspace*{\algorithmicindent} 1. List of statistically significant regional colocation patterns $[<r_{g}, C>]$ \\
\hspace*{\algorithmicindent}\textbf{Variables:}\\
\hspace*{\algorithmicindent} Distance between feature instances $d$ 
\end{flushleft}
\vspace{-10pt}
\begin{algorithmic}[1]
    \Procedure{Statistically Significant Regional-Colocation Miner}{}
    \ForEach{$f_{k}$ in \{$f_{A}, f_{B}, ...$\}}
    \State {Generate $R$ null hypotheses ($NH_{\emptyset}$) using summary statistics in each $r_{g} \in R_{g}$.}
    \EndFor
    \ForEach{candidate pattern $C_{m} \in \{C_{1}, C_{2},...,C_{M}\}$}
    \For{distance $d$ $\in$ [LB, LB + 10,..., UB]}
        \ForEach{$r_{g} \in R_{g}$}
            \State{$SSR_{C_m}^{r_g}$, $p-value$ $\leftarrow$ Significance Testing(S, $r_{g}$, $\alpha$, $C_{m}$, $NH_{\emptyset}$, $d$)}
            
            \If{$SSR_{C_m}^{r_g}$ is True}
                \State{Insert $r_{g}$ in significant atomic partitions list}
            \EndIf
        \EndFor
        \State{Compose Neighborhood graph ($G$) from significant atomic partitions list }
        
        \State{$r_{g}^{final} \leftarrow r_{g}^{maxPI}$} \Comment{atomic partition in $G$ with highest $pi$}
        \ForEach{$r_{g}$ $\in$ Depth First Graph Traversal of $G$ \\ \hfill from vertices adjacent to $r_{g}^{maxPI}$} \Comment{$r_{g} \neq r_{g}^{maxPI}$}
            \State{$r_{g}^{temp} \leftarrow r_{g}^{final} \cup r_{g}$}
            \State{$SSR_{C_m}^{r_g}$, p-value $\leftarrow$ Significance Testing($S$, $r_{g}^{temp}$, $\alpha$, $C_{m}$, $NH_{\emptyset}$, $d$)}
            \If{$SSR_{C_m}^{r_g}$ is True}
                \State{$r_{g}^{final} \leftarrow r_{g}^{temp}$}
            \EndIf
        \EndFor
        \State{Add $<r_{g}^{final}, C_{m}>$ to $[<r_{g}, C>]$}
    \EndFor
\EndFor
\State{\textbf{return} $[<r_{g}, C>]$}
\EndProcedure
\end{algorithmic}
\end{algorithm}
\setlength{\textfloatsep}{5pt}


\subsection{Multiple Comparisons Regional Colocation Miner (MultComp-RCM)}
\label{sec:proposed}
\textbf{Key Idea:}
The baseline $SSRCM$ computes a significance test for every union of statistically significant partitions, resulting in many participation index $(pi)$ computations and significance tests. We address this by using a Bonferroni correction, which selects atomic partitions conservatively, increasing the chances that their union is also statistically significant. The Bonferroni correction reduces the need for a regional statistical significance test for each union operation.

A Bonferroni correction is used when several independent statistical inferences are being performed simultaneously. Although a given significance threshold  $(\alpha)$ on the $p$-$value$ may be appropriate for an individual test, it is not sufficient for the set of all comparisons. To reduce many false positives, the $\alpha$ needs to be lowered to account for the number of comparisons performed. The Bonferroni correction sets the statistical significance threshold for the entire set of $n$ comparisons to $\alpha/n$ or by multiplying the $p$-$value$ by $n$ and then applying the standard threshold $\alpha$. This conservative correction works even under the most extreme circumstances (e.g. when all $n$ tests are independent of one another).

We check for statistical significance in each input partition in the proposed approach. Then, we perform a graph traversal starting from the atomic partition with the highest $pi$ value for the candidate regional-colocation pattern. Then, instead of recomputing the $pi$ and testing the candidate pattern in the new bigger region (composed of atomic partitions) for statistical significance, we perform a Bonferroni correction. Thus, if we were initially checking for a threshold level of $0.05$, we would be checking for a threshold level of $0.05/2$ in each atomic partition for the union of two partitions. This conservative threshold reduces Type-I error by returning regions with much higher statistical confidence. The union of the atomic partitions is sequential, and every atomic partition must satisfy the adjusted p-value threshold to be considered for the union.

Algorithm \ref{algorithm3} provides a snippet of MultComp-RCM showing the use of the Bonferroni correction. Lines $13$-$18$ show the new steps in the refined approach.

\begin{algorithm}[h]
\caption{MultComp-RCM snippet} \label{algorithm3}
\footnotesize
\begin{flushleft}
\end{flushleft}
\begin{algorithmic}[1]
\vspace{-25pt}
    \Procedure{MultComp-RCM}{}
    \State{$\vdots$}
     \setcounter{ALG@line}{11}
        \State{$n$  $\leftarrow$ 1} \Comment{Number of atomic partitions in the region.}
       
        \ForEach{$r_{g}$ $\in$ Depth First Graph Traversal of $G$ \\ \hfill from vertices adjacent to $r_{g}^{maxPI}$} \Comment{$r_{g} \neq r_{g}^{maxPI}$}
            \State{Subgraph ($SG$) $\leftarrow$ $r_{g}^{final}$ $\cup$ $r_{g}$}
            \State{flag $\leftarrow$ 1}
            \State{flag = BONF_CHECK(flag , $SG$, n)}
            \If{flag == 1}
                \State{Update $r_{g}^{final}$ $\leftarrow$ $SG$}
                \State{$n \leftarrow n + 1$}
            \EndIf
        \EndFor
        \State{Add $<r_{g}^{final}, C_{m}>$ to $<r_{g}^{C}, C>$}
    \State{$\vdots$}
\EndProcedure
        \Procedure{Bonf_Check}{flag , $SG$, n}
        \State{$p$-$value_{threshold}$ $\leftarrow$ $\alpha/(n+1)$}
                \ForEach{$node \in SG$}
                    \If{$p-value$ of pattern $C_{m}$ in $node$ $\not \leq$ $p$-$value_{threshold}$}
                        \State{flag $\leftarrow$ 0}
                    \EndIf
                \EndFor
            \State{\textbf{return} flag}
        \EndProcedure
\end{algorithmic}
\end{algorithm}
\setlength{\textfloatsep}{5pt}

Figure \ref{fig:toy_example} shows an execution trace of merging $4$ neighboring partitions. Each region has a participation index and a $p$-$value$ computed individually. Then, Steps 1-4 show the process of combining these partitions based on either additional statistical significance tests and participation index computations (for SSRCM) or using a tighter $p$-$value$ threshold for MultComp-RCM. Table \ref{tab:compare} compares the number of computations for the two approaches and clearly shows the lower computational requirements of MultComp-RCM. When performing the union of two regions using MultComp-RCM, the new threshold as per the Bonferroni correction is applied to each of the two regions (as in procedure BONF\_CHECK in Algorithm \ref{algorithm3}) for a successful union.
\begin{figure}[htp]
\centering
\graphicspath{ {./images/} }
{\includegraphics[width=0.50\linewidth]{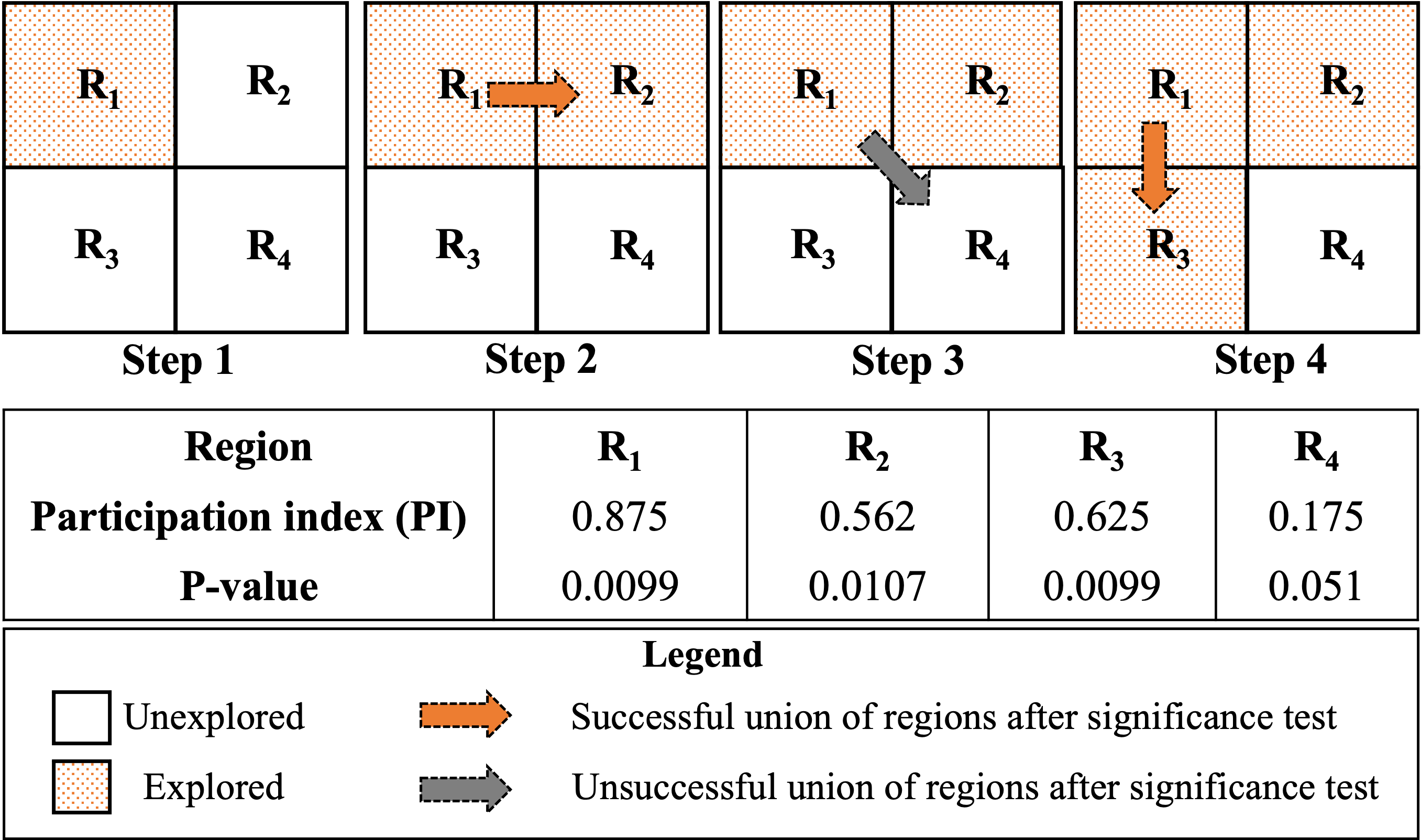}\label{fig:comparison_steps}}
\caption{Execution trace of SSRCM and  MultComp-RCM.}\label{fig:toy_example}
\end{figure}

\begin{table}[h]
\footnotesize
\caption{Comparing the cumulative number of statistical significance tests (C\#), participation index computation (pi cal.), and p-value thresholds (p-val th.) between $SSRCM$ (denoted as $S$) and MultComp-RCM (denoted as $R$).}\label{tab:compare}
\begin{tabular}{|p{0.60cm}|p{0.65cm}|p{0.65cm}|p{1.2cm}|p{1.2cm}|p{1.65cm}|p{1.65cm}|}
\hline 
Steps & C\#$S$ & C\#$R$ & \emph{pi} cal. S & \emph{pi} cal. R & \emph{p-val.} th. $S$ & \emph{p-val.} th. $R$\tabularnewline
\hline 
0 & 4 & 4 & 4 & 4 & 0.05 & 0.05\tabularnewline
\hline 
1 & 5 & 4 & 5 & 4 & 0.05 & 0.05\tabularnewline
\hline 
2 & 6 & 4 & 6 & 4 & 0.05 & 0.025 \tabularnewline
\hline 
3 & 7 & 4 & 7 & 4 & 0.05 & 0.0167\tabularnewline
\hline 
4 & 8 & 4 & 8 & 4 & 0.05 & 0.0125\tabularnewline
\hline 
\end{tabular}
\end{table}

\section{Theoretical Analysis}
\label{sec:theoretical_analysis}
\begin{lemma}
 $MultComp-RCM$ has lower or equal Type-I error than $SSRCM$.
\end{lemma}
\begin{proof}
Algorithm \ref{algorithm1} called by Algorithm \ref{algorithm2} and \ref{algorithm3} in line 7 extracts atomic partitions within which a regional-colocation pattern is statistically significant.

The Bonferroni correction in procedure BONF\_CHECK in Algorithm \ref{algorithm3} controls the experiment-wide false positive rate ($\pi$) by specifying the significance level ($\alpha$) for each test, where a test is significant if $p-value\leq\alpha$. The probability of no Type $I$ error (false positives) in $n$ independent tests is $(1-\alpha)^{n}$, if each test is at level $\alpha$. Therefore, the probability of at least one false positive $\pi$ is $1-(1-\alpha)^{n}$. For an experiment-wide false positive rate of $\pi$, the $\alpha$ for each test should be $\alpha=1-(1-\pi)^{1/n}$.
Using binomial approximation, $(1-\alpha)^{n}\simeq1-n\alpha$, which gives $\alpha=\pi/n$. For an experiment-wide false positive value $\pi=0.05$, the $\alpha$ (false positive rate for each test) should be less than $\pi,$ i.e. $\alpha\leq\pi$. Therefore each region and sub-region output by MultComp-RCM has lower Type-I and a precision close to $1$.
\end{proof}


\begin{lemma}
\label{lemma:Cost} 
MultComp-RCM has lower or equal computational cost than $SSRCM$ for all observed data, where Bonferroni-revised p-values eliminate lower confidence candidates considered by the original p-value, i.e. $Cost_{MultComp-RCM} \leq Cost_{SSRCM}$.
\end{lemma}
\begin{proof}
Let $C_{pi}(d)$ be the complexity of participation index $(pi)$ computation for a specific region (dependent on the data $d$). Let $C_{st}(pi_{obs}, pi_{null}, d)$ be the complexity of significance testing for a specific region (dependent on the $pi$ in observed data $d$ and the null hypothesis). Assume $N_{1}$ is the number of space partitions/regions in the dataset, $N_{2}$ is the number of space partitions extracted from Algorithm \ref{algorithm1}, and $N_{2}\leq N_{1}$. Further, assume $d_{1}$ is the initial dataset and $d_{2}$ is the dataset in each iteration in the SSRCM. Then, the cost of $SSRCM$ is  
\begin{equation}
    N_{1}(C_{pi}(d_{1})+C_{st}(pi_{obs},pi_{null},d_{1}))+N_{2}(C_{pi}(d_{2})+C_{st}(pi_{obs},pi_{null},d_{2}))
\end{equation}
By contrast, the cost of the proposed MultComp-RCM approach is only $N_{1}(C_{pi}(d_{1})+C_{st}(pi_{obs},$ $pi_{null},d_{1}))+N_{2}$.
Here, $N_{2}$ represents the number of significant partitions for which the p-value needs a comparison against the threshold obtained from Bonferroni correction.
\end{proof}

\section{Experimental Evaluation}
\label{sec:Experiments}
We had three goals for the experiments: (1) To compare the time taken by $SSRCM$ and MultComp-RCM with varying numbers of regional-colocation instances, varying number of atomic partitions, and change in the number of feature instances. (2) To compare the number of significance tests, $pi$ calculations for a varying number of regions. (3) To compare solution quality between $SSRCM$ and MultComp-RCM. 

\textbf{Experiment design:} Figure \ref{fig:validation_framework} shows the overall validation framework. The metric for comparing the solution quality of $SSRCM$ with MultComp-RCM was the false positive rate (FPR), while the runtime comparisons were based on the execution time (in seconds) of the individual algorithms. The experiments were done on both real (Safegraph POI) and synthetic data to perform both comparative and sensitivity analysis.
\begin{figure}[htbp]
  \centering
  \graphicspath{ {./images/} }
  \includegraphics[width=0.80\linewidth]{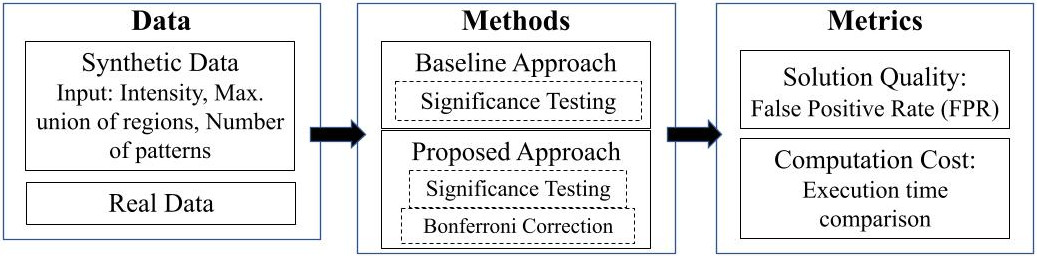}
  \caption{Overall validation framework}
  \label{fig:validation_framework}
\end{figure}

\textbf{Synthetic data generation:} We began with a space partitioning $(R_{g})$, a maximum union (or traversal) of regions $(L_{max})$, and a number of regional-colocation patterns i.e., pairs of $<r_{g},C>$. We then generated reference points within the partitions using the Poisson point process. At each reference point, we generated circles of diameter $d_g$ which was determined empirically for each region in $R_{g}$ in the observed dataset. The diameter signifies the smallest distance between features in a colocation $C$ at which they become statistically significant regional colocations. We populated each
circle with instances of $C$. We note that the circles were only used to place colocated instances in a region and were not separate partitioning. Figure \ref{fig:syn_data_gen} shows the process of synthetic data generation.

\begin{figure}[htbp]
  \centering
  \graphicspath{ {./images/} }
  \includegraphics[width=0.60\linewidth]{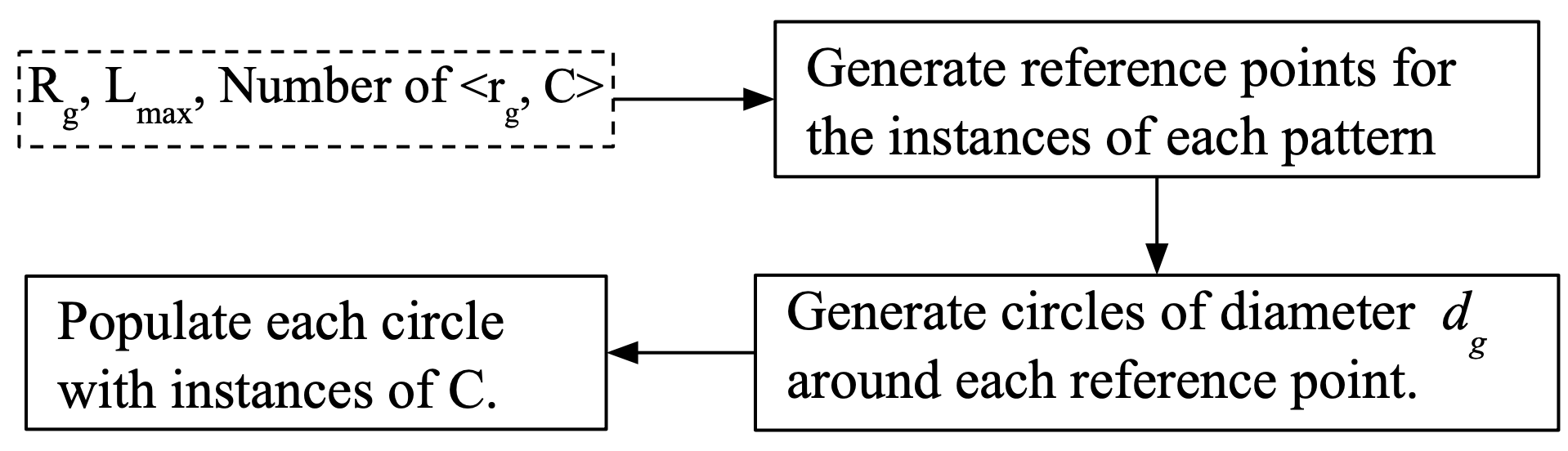}
  \caption{Synthetic data generation process.}
  \label{fig:syn_data_gen}
\end{figure}

\textbf{Comparative Analysis:} Figure \ref{subfig:syn1} shows the time taken (in log scale) for different regional-colocation instances. For this experiment, we varied the number of regional-colocation instances in each atomic partition from $4$ to $84$ while keeping other parameters (like the number of regions) constant and record the execution time of both algorithms. Figure \ref{subfig:syn2} compares the execution time with a varying number of atomic partitions (or regions) while keeping the number of regional-colocation instances in each partition constant. Figure \ref{subfig:syn3} shows the time taken with a varying number of feature instances (which constitute the regional-colocation pattern) in each region while keeping the number of regions constant. In all experiments, $MultComp$-$RCM$ is much faster than the baseline $SSRCM$. These results are consistent with Lemma \ref{lemma:Cost}, which says that $Cost_{MultComp-RCM}\leq Cost_{SSRCM}$.





\begin{figure*}[h]
\centering
\subfloat[Number of reg. col. instances]{\includegraphics[width = 4.2cm]{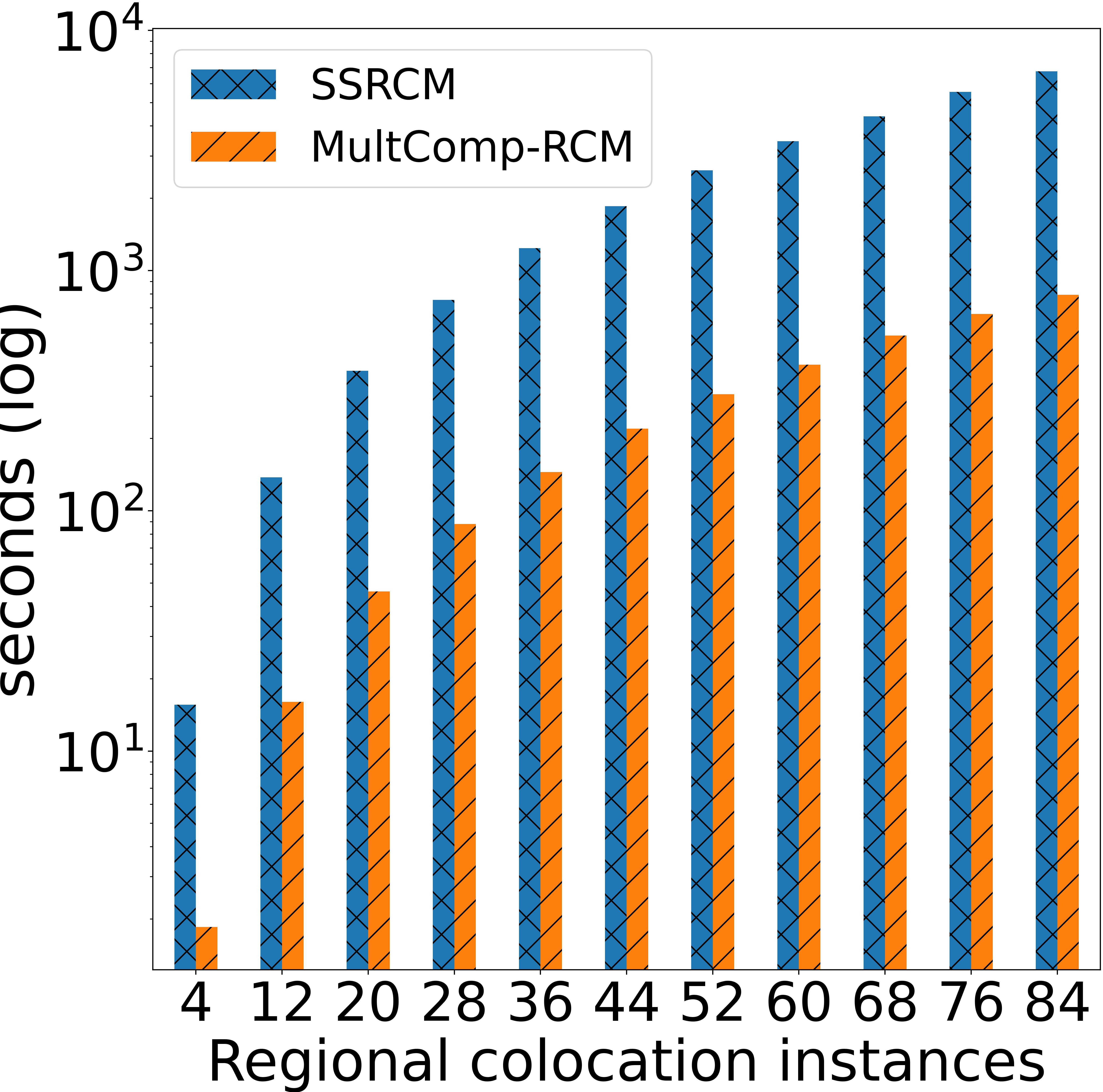}\label{subfig:syn1}}\hfill
\subfloat[Number of regions]{\includegraphics[width = 4.3cm]{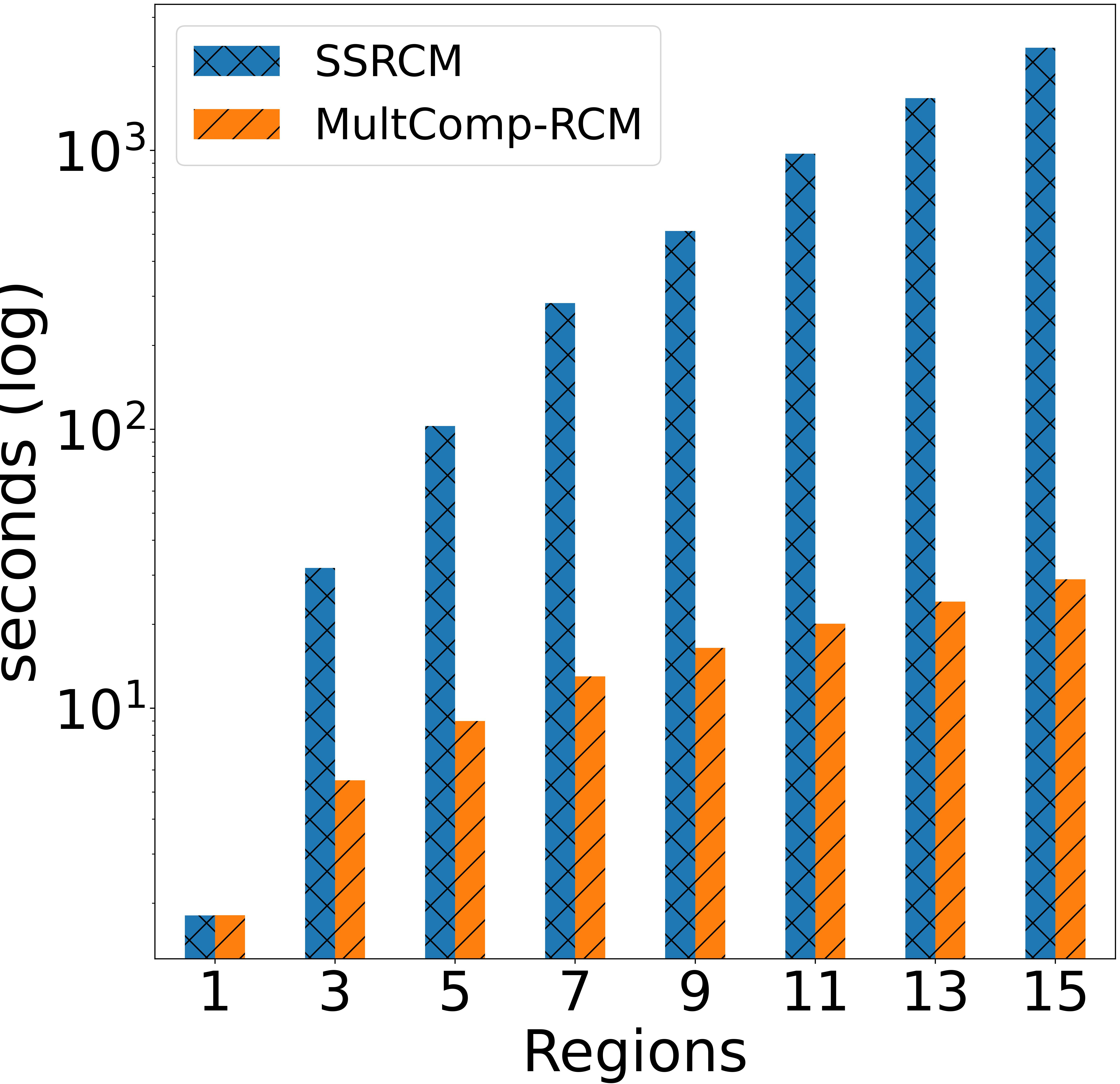}\label{subfig:syn2}}\hfill
\subfloat[Number of feature instances]{\includegraphics[width = 4.2cm]{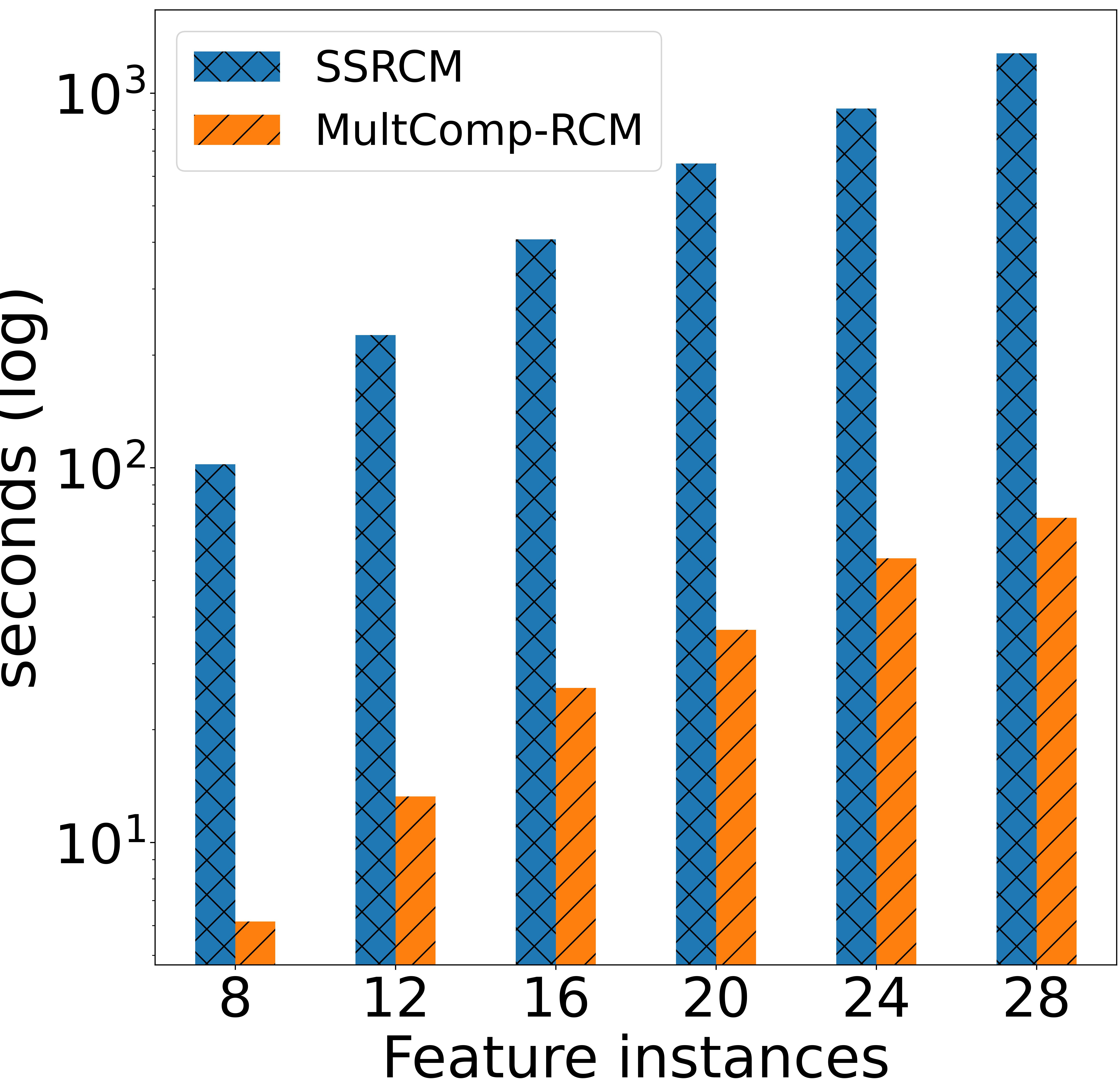}\label{subfig:syn3}}\vfill
\subfloat[Number of significance tests]{\includegraphics[width = 4.2cm]{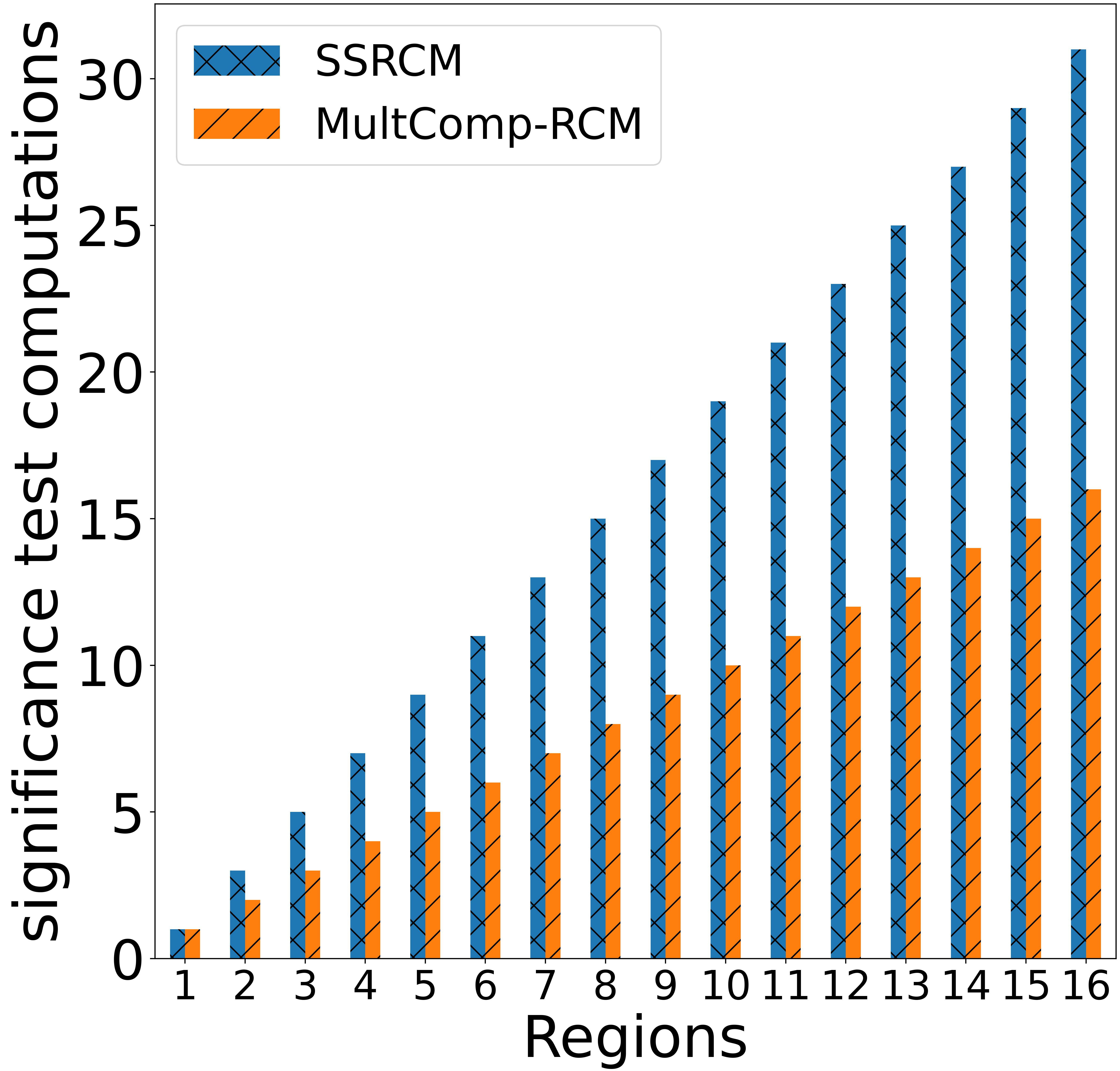}\label{subfig:syn4}}\hspace{0.1\textwidth}
\subfloat[Number of PI calculations]{\includegraphics[width = 4.3cm]{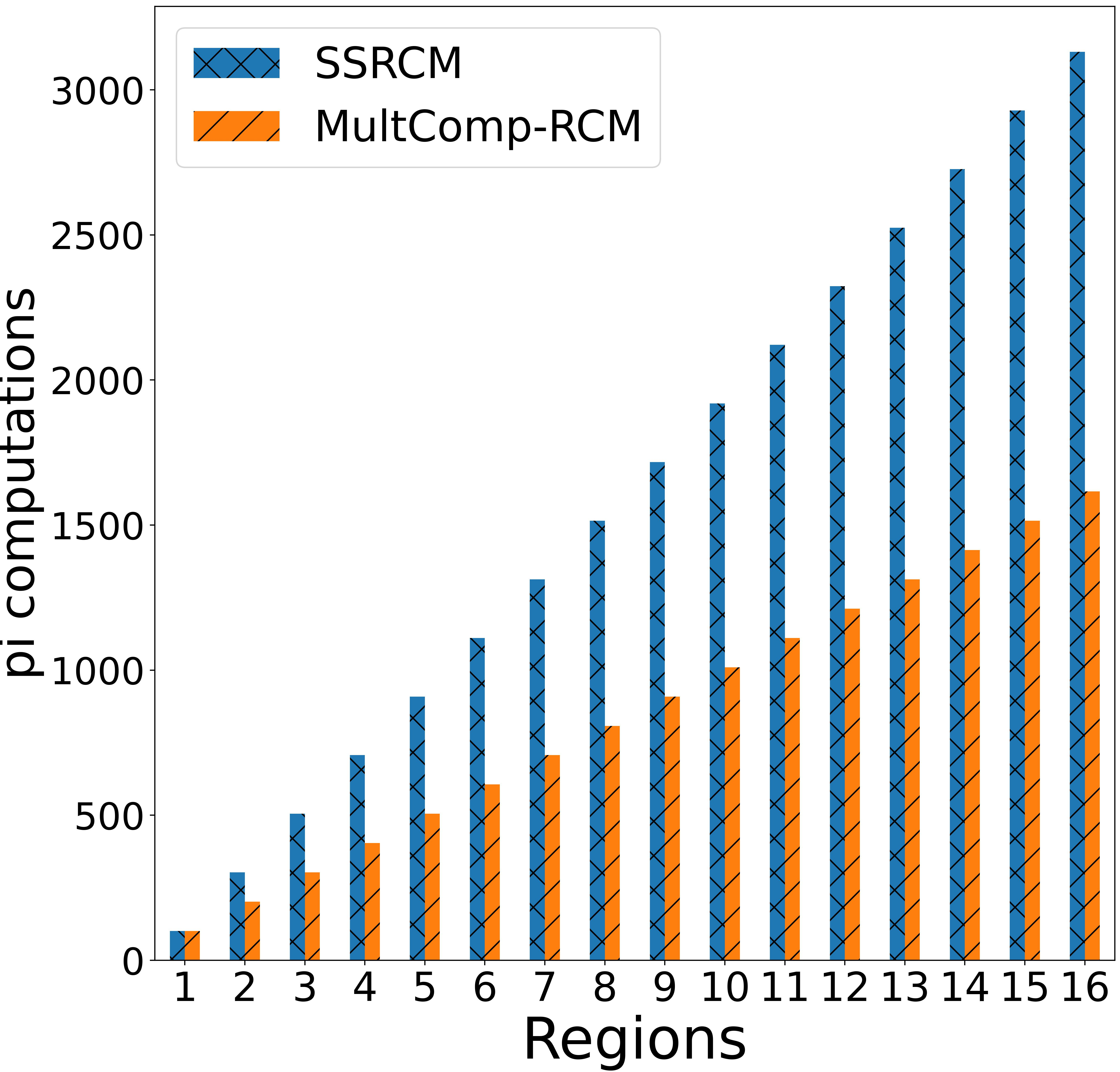}\label{subfig:syn5}}

\caption{MultComp-RCM outperforms SSRCM \cite{10.1145/3557989.3566158}}
\label{fig:sensitivity}
\end{figure*}

\textbf{Sensitivity Analysis:} Figure \ref{subfig:syn4} shows the number of significance tests performed by both algorithms with varying number of regions, while keeping the number of regional-colocation instances constant in each partition. Figure \ref{subfig:syn5} shows the number of participation index computations performed with varying number of regions with the same constant parameters as above. In both cases, the proposed MultComp-RCM requires lesser number of significance tests and participation index computations for an increasing number of regions.

\textbf{Solution Quality:} We performed controlled experiments on synthetic datasets to compare the solution quality of $MultComp$-$RCM$ with $SSRCM$. Metric for comparison was the false positive rate (FPR).

\noindent $FPR=\frac{FP}{FP+TN}$, where $FP$ is the number of false positives, and $TN$ is the number of true negatives. Table \ref{tab:sol_qual} shows the experiment results. As shown MultComp-RCM exhibits a lower rate for false pattern discovery than $SSRCM$. This is mainly because MultComp-RCM eliminates regions which barely pass the atomic significance test (borderline statistical confidence) in Algorithm \ref{algorithm1}, which SSRCM fails to reject in the final output.

\begin{table}[tbhp]
\centering
{
\footnotesize
\caption{MultComp-RCM generates less false positives}\label{tab:sol_qual}
\begin{tabular}{|p{1.0cm}|p{4.2cm}|p{4.5cm}|}
 \hline
 \textbf{Pattern} & \textbf{SSRCM False Positive Rate} & \textbf{MultComp-RCM False Positive Rate} \\ \hline
 {A, B, C} & 0.15 & 0.03\\\hline
 {A, B} & 0.17 & 0.01\\ \hline
 {B, C} & 0.14 & 0.01\\ \hline
 {A, C} & 0.19 & 0.04\\ \hline
\end{tabular}
}
\end{table}

\section{Case Study}
\label{subsec:case-study}
We extended our previous case study \cite{10.1145/3557989.3566158} to show the effectiveness of the proposed approach. 

\noindent\textbf{Dataset}: We used data from SafeGraph, a mobility data vendor who provides anonymized aggregated location data to researchers studying the effects of COVID-19 on citizen mobility patterns towards numerous Points Of Interest (POIs). The dataset consists of 1473 retail brands in Minnesota. Experiments were performed on colocation patterns consisting of two (e.g., Jimmy John’s, McDonald’s) or three (e.g., Jimmy John’s, McDonald’s, Subway) features. Our null hypothesis generation followed the procedure described in Section \ref{sec:stat_sig_test_basic_concepts}.

\noindent \textbf{Case Study Results:} The pattern $C := \{Jimmy John's, McDonald's, Subway\}$ was found to be statistically significant when the distance between feature instances was about $1400$ meters. The regional footprint was the union of `Dakota' and `Hennepin' Counties. The $pi$ values in the counties were $0.34$ and $0.45$ respectively. The $p$-$value$ for the pattern within the counties were $0.02$ and $0.01$, satisfying the $p$-$value$ threshold of $\frac{0.05}{2}$ as per the Bonferroni correction for the two partitions. A few additional significant patterns are shown in Table \ref{tab:case_study_results} (values rounded to two decimal places).


\begin{table}[tbhp]
{
\centering
\footnotesize
\caption{Regional-colocation patterns found to be statistically-significant at distance $d$.}\label{tab:case_study_results}
\begin{tabular}{|p{5.3cm}|p{6.5cm}|p{0.95cm}|}
 \hline
 \textbf{Colocated features} & \textbf{Counties ($participation index,p-value$)} & $d$ \\ \hline
 \{ Caribou coffee, Starbucks\} & Hennepin ($0.34,0.01$) & $200$ m \\\hline
 \{ Caribou coffee, Starbucks\} & Carver ($0.5,0.02$), Hennepin ($0.51,0.01$), Washington ($0.41,0.01$) & $400$ m \\\hline
  \{ Caribou coffee, Starbucks, Dunn Bros\} & Hennepin ($0.52,0.01$) & $1900$ m \\\hline
  \{ Caribou coffee, Starbucks, Dunn Bros\} & Hennepin ($0.72,0.01$), Washington ($0.36,0.02$) & $3000$ m \\\hline
  \{ Jimmy John's, McDonald's\} & Hennepin ($0.39,0.01$)& $500$ m \\\hline
  \{ Jimmy John's, McDonald's\} & Dakota ($0.36,0.02$), Hennepin ($0.51,0.01$) & $700$ m \\\hline
  \{ Jimmy John's, McDonald's, Subway\} & Dakota ($0.34,0.02$), Hennepin ($0.45,0.01$) & $1400$ m \\\hline
  \{ Jimmy John's, McDonald's, Subway\} & Dakota ($0.47,0.02$), Hennepin ($0.57,0.01$), Washington ($0.43,0.02$) & $1500$ m \\\hline
\end{tabular}
}
\end{table}

In our previous paper \cite{10.1145/3557989.3566158}, we compared SSRCM with the Quad and QGFR algorithms \cite{li2018local} whose data-aware space partitioning approach is based on the minimum orthogonal bounding rectangle ($MOBR$). We found that the MOBR-based approach with a participation index threshold of $0.6$ produced $3368$ potential localities for the pattern \{$r_{g}$, [Caribou Coffee, Starbucks]\}. With a confidence level of $95\%$, MOBR-based approach resulted in $2917$ significant and $451$ non-significant patterns. Hence, a regional-colocation miner without statistical significance may enumerate output regions where colocations occurred by chance.


\section{Related Work and Discussion}
\label{sec:related}
\noindent \textbf{Related Work:} The concept of colocation was introduced by Shekhar et al. \cite{shekhar2001discovering}. Huang et al. \cite{huang2004discovering} provided extensive experiments and rigorous discussions regarding the topic and the participation index as a prevalence measure between constituent features. Later, Barua et al. \cite{barua2013mining} introduced statistical significance testing in global colocation and segregation pattern detection to avoid enumeration of chance patterns in the dataset for both aggregation and segregation patterns but did not mention patterns that are regional (or local). 
Regional colocation with minimum orthogonal bounding rectangle (MOBR) based approach was studied by Li et al. \cite{li2018local} while \cite{Wang2013RegionalCO} and \cite{Celik2007ZonalCP} focused on shapes and zonal patterns, respectively. These methods utilized a threshold on the participation index ($pi$) without statistical significance testing, leading to the detection of spurious patterns (as discussed in \cite{10.1145/3557989.3566158}). We \cite{10.1145/3557989.3566158} recently proposed a subgraph-based approach that incorporates statistical significance in detecting regional colocation patterns. This approach reduced the number of spurious patterns detected by previous methods. However, due to a large number of simultaneous statistical inferences, an increase in false discoveries is also observed.
Besides, other patterns \cite{sharma2022towards} and several statistical significance and false discovery reduction techniques have been studied in association rule mining \cite{webb2007discovering}, \cite{duivesteijn2011exploiting}. However, these approaches do not address the inherent variability in spatial data (i.e., different summary statistics of features in each atomic partition). To find subgroups of items, which are generally observed to be statistically significant associations, they compare a quality measure (which assigns to each itemset a numeric value) on the subgroup against that in a statistical model (which corresponds to the null hypothesis). These null hypotheses for significance testing are uniform and do not address spatial variability. Thus these approaches are not directly applicable to regional colocation patterns (more details in Appendix \ref{Appendix:localPatternMining}).
\vspace*{-1.0mm}
\section{Conclusion and Future Work}
\label{sec:Conclusion}
In this paper, we refined the problem of the statistically significant regional-colocation pattern ($SSCRP$). We proposed a robust $MultComp$-$RCM$ approach that reduces the number of false positives using a Bonferroni correction. We theoretically show that MultComp-RCM has a lower or equal Type-I error and computational cost than $SSRCM$ along with experimental results. We extended the previous case study on retail establishments in Minnesota using the proposed approach showing a contrast between significant and non-significant patterns.\\
\noindent \textbf{Future Work:} We plan to explore other methods to reduce Type-I errors (false positives) while also addressing Type-II errors (false negatives) arising from the conservative Bonferroni correction approach and further add temporal dimension to these patterns.

\bibliography{document}

\begin{thebibliography}{10}

\bibitem{barua2013mining}
Sajib Barua and J{\"o}rg Sander.
\newblock Mining statistically significant co-location and segregation patterns.
\newblock {\em IEEE TKDE}, 26(5):1185--1199, 2013.

\bibitem{besag1977simple}
Julian Besag and Peter~J Diggle.
\newblock Simple monte carlo tests for spatial pattern.
\newblock {\em Journal of the Royal Statistical Society: Series C (Applied Statistics)}, 26(3):327--333, 1977.

\bibitem{bonferroni1936teoria}
Carlo Bonferroni.
\newblock Teoria statistica delle classi e calcolo delle probabilita.
\newblock {\em Pubblicazioni del R Istituto Superiore di Scienze Economiche e Commericiali di Firenze}, 8:3--62, 1936.

\bibitem{Celik2007ZonalCP}
M~Celik et~al.
\newblock Zonal co-location pattern discovery with dynamic parameters.
\newblock {\em ICDM}, 2007.

\bibitem{Deng2017MultilevelMF}
Min Deng et~al.
\newblock Multi-level method for discovery of regional co-location patterns.
\newblock {\em IJGIS}, 2017.

\bibitem{duivesteijn2011exploiting}
Wouter Duivesteijn and Arno Knobbe.
\newblock Exploiting false discoveries--statistical validation of patterns and quality measures in subgroup discovery.
\newblock In {\em 2011 IEEE 11th International Conference on Data Mining}, pages 151--160. IEEE, 2011.

\bibitem{Eick2008FindingRC}
Christoph~F. Eick, Rachana Parmar, et~al.
\newblock Finding regional co-location patterns for sets of continuous variables in spatial datasets.
\newblock In {\em SIGSPATIAL}, 2008.

\bibitem{10.1145/3557917.3567619}
Yan~Li et~al.
\newblock Cscd: Towards spatially resolving the heterogeneous landscape of mxif oncology data.
\newblock In {\em Proceedings of the 10th ACM SIGSPATIAL International Workshop on Analytics for Big Geospatial Data}, BigSpatial '22, pages 36--46, New York, NY, USA, 2022. ACM.

\bibitem{10.1145/3557989.3566158}
Subhankar Ghosh, Jayant Gupta, Arun Sharma, Shuai An, and Shashi Shekhar.
\newblock Towards geographically robust statistically significant regional colocation pattern detection.
\newblock In {\em Proceedings of the 5th ACM SIGSPATIAL International Workshop on GeoSpatial Simulation}, GeoSim '22, page 11–20, New York, NY, USA, 2022. Association for Computing Machinery.
\newblock \href {https://doi.org/10.1145/3557989.3566158} {\path{doi:10.1145/3557989.3566158}}.

\bibitem{huang2004discovering}
Yan Huang et~al.
\newblock Discovering colocation patterns from spatial data sets: a general approach.
\newblock {\em IEEE TKDE}, 16(12):1472--1485, 2004.

\bibitem{illian2008statistical}
Janine Illian, Antti Penttinen, et~al.
\newblock {\em Statistical analysis and modelling of spatial point patterns}, volume~70.
\newblock John Wiley \& Sons, 2008.

\bibitem{li2018local}
Yan Li and Shashi Shekhar.
\newblock Local co-location pattern detection: a summary of results.
\newblock In {\em GIScience}. Schloss Dagstuhl-Leibniz-Zentrum fuer Informatik, 2018.

\bibitem{risse1992long}
Guenter~B Risse.
\newblock ``a long pull, a strong pull, and all together'': San francisco and bubonic plague, 1907-1908.
\newblock {\em Bulletin of the History of Medicine}, 66(2):260--286, 1992.

\bibitem{rupert2012simultaneous}
G~Rupert~Jr et~al.
\newblock {\em Simultaneous statistical inference}.
\newblock Springer Science \& Business Media, 2012.

\bibitem{sharma2022towards}
Arun Sharma, Jayant Gupta, and Subhankar Ghosh.
\newblock Towards a tighter bound on possible-rendezvous areas: preliminary results.
\newblock In {\em Proceedings of the 30th International Conference on Advances in Geographic Information Systems}, pages 1--11, 2022.

\bibitem{shekhar2001discovering}
Shashi Shekhar and Yan Huang.
\newblock Discovering spatial co-location patterns: A summary of results.
\newblock In {\em Intl. symposium on spatial and temporal databases}, pages 236--256. Springer, 2001.

\bibitem{Wang2013RegionalCO}
Song Wang et~al.
\newblock Regional co-locations of arbitrary shapes.
\newblock In {\em SSTD}, 2013.

\bibitem{webb2007discovering}
Geoffrey~I Webb.
\newblock Discovering significant patterns.
\newblock {\em Machine learning}, 68(1):1--33, 2007.

\bibitem{wong2004modifiable}
David~WS Wong.
\newblock The modifiable areal unit problem (maup).
\newblock In {\em WorldMinds: geographical perspectives on 100 problems: commemorating the 100th anniversary of the association of American geographers 1904--2004}, pages 571--575. Springer, 2004.

\end{thebibliography}
\newpage
\appendix
In this appendix, we address the following questions:
\vspace*{-1.0mm}
\section{Why can't we use existing false discovery reduction techniques from local pattern mining?}
\label{Appendix:localPatternMining}
Existing techniques for reducing false discoveries in local pattern mining cannot be applied to this problem, because of spatial variability (i.e. constituent features of a regional colocation pattern might have different summary statistics in different atomic partitions). Webb \cite{webb2007discovering} proposed a holdout approach where one divides the data into exploratory and holdout sets. Patterns are generated using the exploratory data, while statistical tests are performed on the generated patterns using the holdout data. This technique may apply to atomic partitions with a large presence of constituent features (e.g., partition T41 in Figure \ref{fig:as_rul_mn_8}). However, it would be counterproductive in partitions where the number of feature instances is very low (e.g., partition T42 in Fig. \ref{fig:as_rul_mn_8}). In such partitions splitting the data points into exploratory and holdout sets would result in very few instances for the pattern detection process.
\begin{figure}[htbp]
  \centering
  \graphicspath{ {./images/} }
  \includegraphics[width=0.30\linewidth]{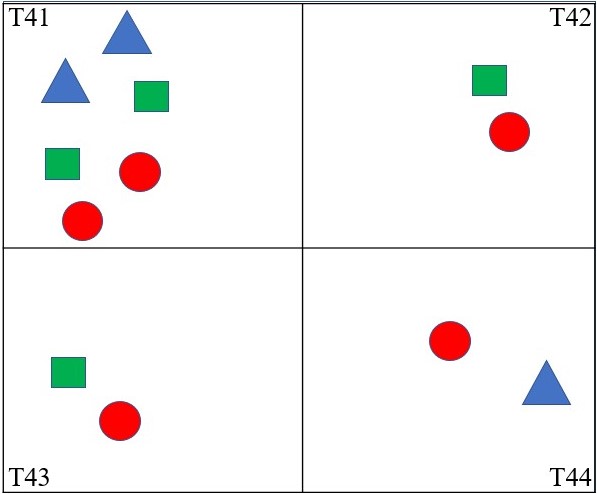}
  \caption{Feature instances exhibit spatial variability within atomic partitions.}
  \label{fig:as_rul_mn_8}
\end{figure}
\section{Why can't this problem be cast as a modified version of frequent itemset mining?}
In frequent itemset mining, the task is to find subgroups of items that often occur together in a transaction, e.g., laptop and antivirus software. Previous works have been done on addressing false discoveries in this problem \cite{duivesteijn2011exploiting}. Such approaches assign the association in the mined subgroup as the alternate hypothesis while the null hypothesis is formulated using a randomized baseline subset. Thus these approaches do not address the independent relationship between hypotheses in different spatial partitions in our problem. As noted earlier, in regional colocation pattern detection, different features might have different summary statistics in different atomic partitions. To model the complete spatial randomness of these features, we generate the null hypotheses in each atomic partition as per the summary statistics of the said features in that specific partition. Thus the null hypothesis generated for the features in one atomic partition is independent of the null hypothesis in other atomic partitions. Therefore, the problem of regional colocation pattern detection cannot be considered a modified version of subgroup discovery in frequent itemset mining.

\section{How does spatial colocation mining differ from association rule mining?}
Data mining techniques have been widely developed to solve challenging problems in various domains. Yet, the underlying assumption of these algorithms does not address the problem of spatial variability. This leads to the detection of spurious patterns in spatial data, also known as the modifiable aerial unit problem (MAUP \cite{wong2004modifiable}). Colocation pattern detection resembles association rule mining, but the absence of transactions in colocation mining means techniques in association rule mining cannot be used directly to mine colocation patterns.
\begin{figure*}[h]
\centering
\subfloat[Map of 3 spatial feature types]{\includegraphics[width = 3.6cm]{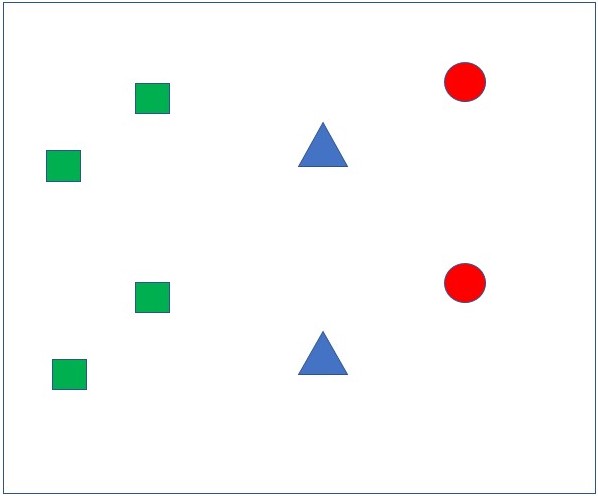}\label{subfig:as_rul_mn_1}}\hfill
\subfloat[Spatial Partition 1]{\includegraphics[width = 3.6cm]{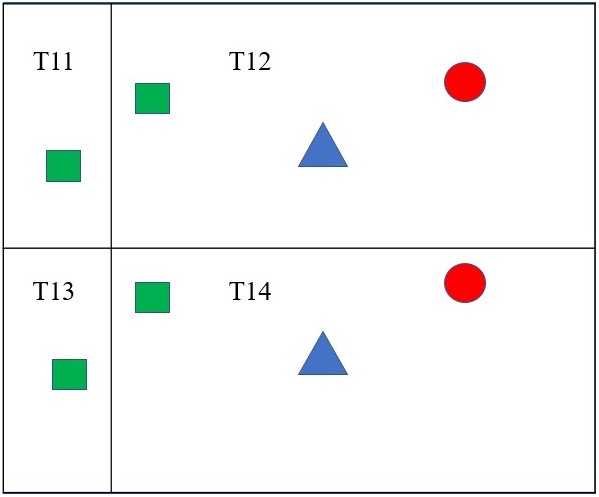}\label{subfig:as_rul_mn_2}}\hfill
\subfloat[Spatial Partition 2]{\includegraphics[width = 3.6cm]{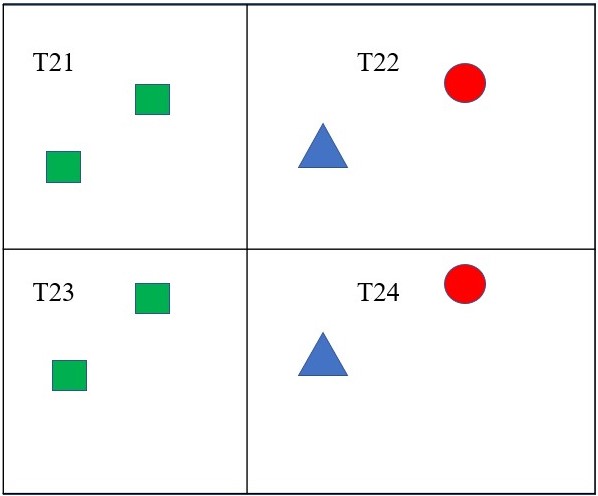}\label{subfig:as_rul_mn_3}}\vfill
\subfloat[Spatial Partition 3]{\includegraphics[width = 3.6cm]{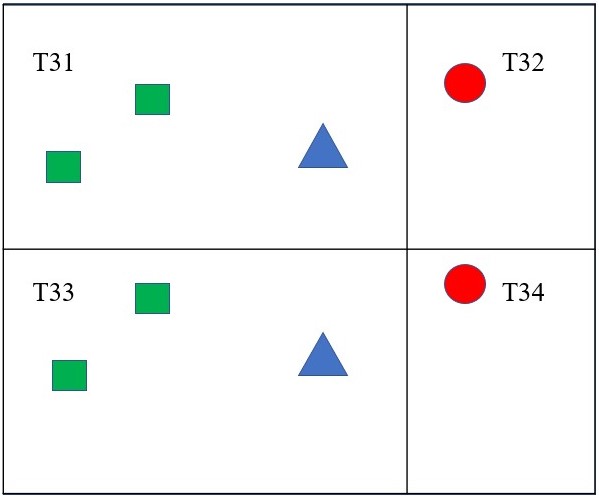}\label{subfig:as_rul_mn_4}}\hfill
\subfloat[Transactions]{\includegraphics[width = 8.9cm]{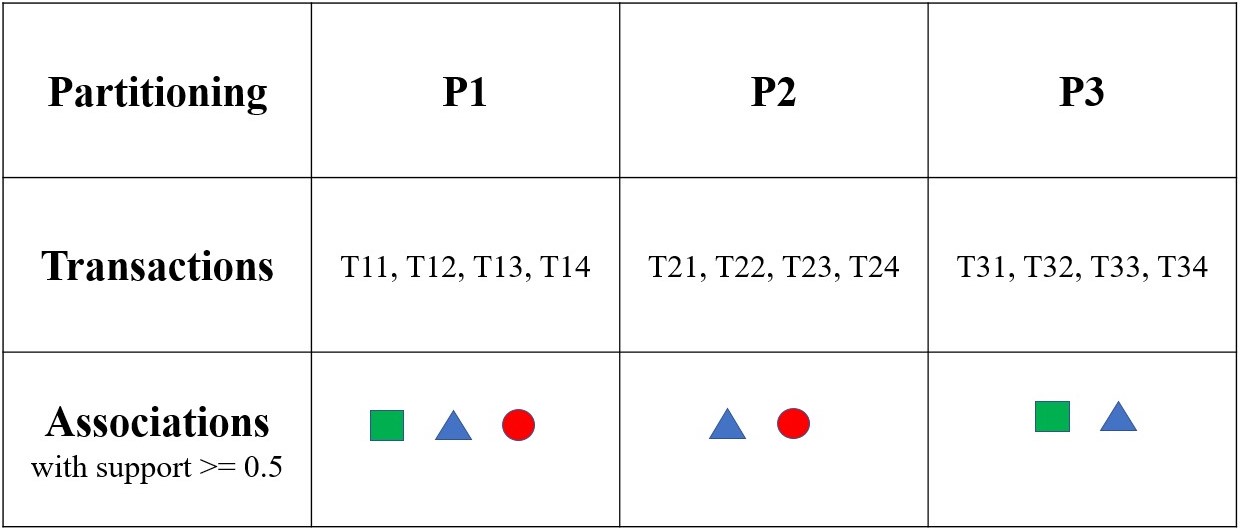}\label{subfig:as_rul_mn_5}}
\caption{Association rule mining \cite{10.1145/3557917.3567619} returning different results depending on the spatial partition}
\label{fig:as_rul_mn}
\end{figure*}

Transactions in association rule mining refer to groups of items purchased together. An itemset's support is the fraction of transactions that contain the itemset. Itemsets greater than a user-specified support value yield to the association rule. In spatial data mining, the choice of partition affects the transaction. For example, Figure \ref{subfig:as_rul_mn_1} below shows a dataset with 3 feature types, i.e. <\textcolor{green}{squares}>, <\textcolor{blue}{triangles}>, <\textcolor{red}{circles}>. In partition P1 (Figure \ref{subfig:as_rul_mn_2}) <\textcolor{green}{squares}, \textcolor{blue}{triangles}, \textcolor{red}{circles}> is a transaction, while in partitions P2 (Figure \ref{subfig:as_rul_mn_3}) and P3 (Figure \ref{subfig:as_rul_mn_4}) <\textcolor{blue}{triangles}, \textcolor{red}{circles}> and <\textcolor{green}{squares}, \textcolor{blue}{triangles}> are the transactions respectively. This is known as the MAUP problem.
In colocation pattern detection this is addressed using a neighborhood graph as shown in Figure \ref{subfig:as_rul_mn_6}. A user-defined neighbor relationship $R$ is used to find subsets of features in close geographic proximity. Thus the colocation miner provides a transaction-free approach to mine prevalent patterns. 

\begin{figure*}[!htb]
\centering
\subfloat[Neighbor graph based on relation $R$]
{\includegraphics[width = 0.25\linewidth]{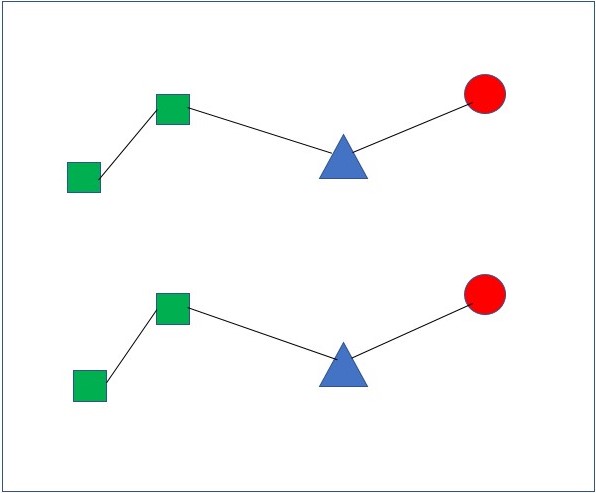}\label{subfig:as_rul_mn_6}}\hfill
\subfloat[PI of candidate patterns]{\includegraphics[width = 0.65\linewidth]{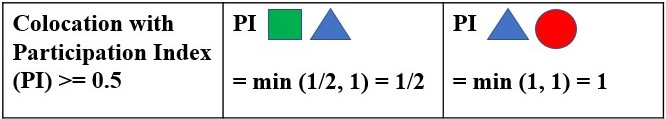}\label{subfig:as_rul_mn_7}}\vfill
\caption{Colocation pattern detection \cite{10.1145/3557989.3566158}}
\label{fig:as_rul_mn_coloc}
\end{figure*}
\end{document}